\documentclass{article}
\usepackage{spconf,amsmath,graphicx}
\usepackage{tikz}
\usepackage{amssymb}
\usepackage{amsthm}
\usepackage{float}
\usepackage{booktabs}
\usepackage{adjustbox}
\usepackage{hhline}
\usepackage{subfig}
\usepackage{bm}
\usetikzlibrary{graphs,graphs.standard}
\graphicspath{{images}}
\usepackage[font={small}, margin=4mm]{caption}
\usepackage{xparse}
\ExplSyntaxOn
\NewDocumentCommand{\longdash}{ O{2} }
 {
  --\prg_replicate:nn { #1 - 1 } { \negthinspace -- }
 }
\ExplSyntaxOff
\newcommand{\vect}[1]{\boldsymbol{#1}}

\def\R{{\mathbb{R}}}
\def\e{{\mathbf{e}}}
\def\M{{\mathit{M}}}
\def\K{{\mathbf{K}}}
\def\S{{\boldsymbol{\Delta}}}
\def\softmax{{\boldsymbol{\sigma}}}
\DeclareMathOperator*{\argmax}{arg\,max}

\DeclareMathOperator*{\logsumexp}{log\,sum\,exp}
\newtheorem{theorem}{Theorem}[section]

\newtheorem{lemma}[theorem]{Lemma}

\newcommand{\irow}[1]{
  \begin{smallmatrix}(#1)\end{smallmatrix}%
}

\usepackage{color,soul}

\usepackage[numbers,sort&compress]{natbib}
\usepackage{hyperref} 

\usepackage{enumitem}
\usepackage{comment}

\title{Implicit Background Estimation for Semantic Segmentation}
\name{Charles Lehman, Dogancan Temel, and Ghassan AlRegib \thanks{}}
\address{Center for Signal and Information Processing,\\ School of Electrical and Computer Engineering,\\ Georgia Institute of Technology, Atlanta, GA, 30332-0250 USA\\ \{charlie.k.lehman,cantemel,alregib\}@gatech.edu}
\begin{document}

\onecolumn 

\begin{description}[labelindent=1cm,leftmargin=3cm,style=multiline]

\item[\textbf{Citation}]{C. Lehman, D. Temel and G. AlRegib, "Implicit Background Estimation for Semantic Segmentation," IEEE International Conference on Image Processing (ICIP), Taipei, Taiwan, 2019.
} \\



\item[\textbf{Code}]{\url{https://github.com/olivesgatech/implicit-background-estimation}} \\

\item[\textbf{Bib}] {
@INPROCEEDINGS\{Lehman2019,\\ 
author=\{C. Lehman and D. Temel and G. AIRegib\},\\ 
booktitle=\{IEEE International Conference on Image Processing (ICIP)\},\\ 
title=\{Implicit Background Estimation for Semantic Segmentation\},\\ 
year=\{2019\},\}\\
} \\

\item[\textbf{Copyright}]{\textcopyright 2019 IEEE. Personal use of this material is permitted. Permission from IEEE must be obtained for all other uses, in any current or future media, including reprinting/republishing this material for advertising or promotional purposes,
creating new collective works, for resale or redistribution to servers or lists, or reuse of any copyrighted component
of this work in other works. } \\

\item[\textbf{Contact}]{\href{mailto:alregib@gatech.edu}{alregib@gatech.edu}~~~~~~~\url{https://ghassanalregib.com/} \\ \href{mailto:charlie.k.lehman@gmail.com
}{charlie.k.lehman@gmail.com~~~~~~~\url{https://charlielehman.github.io/}
} \\ \href{mailto:dcantemel@gmail.com}{dcantemel@gmail.com}~~~~~~~\url{http://cantemel.com/}}
\end{description} 

\thispagestyle{empty}
\newpage
\clearpage

\twocolumn

\maketitle
\begin{abstract}
Scene understanding and semantic segmentation are at the core of many computer vision tasks, many of which, involve interacting with humans in potentially dangerous ways.  It is therefore paramount that techniques for principled design of robust models be developed.  In this paper, we provide analytic and empirical evidence that correcting potentially errant non-distinct mappings that result from the softmax function can result in improving robustness characteristics on a state-of-the-art semantic segmentation model with minimal impact to performance and minimal changes to the code base.
\end{abstract}
\section{Introduction}
\label{sec:intro}
The progress in the semantic segmentation task is in large part thanks to improvements in deep architecture design \cite{DBLP:journals/corr/ChenPK0Y16,DBLP:journals/corr/ChenPSA17,chen2018deeplab, chen2018encoder, NohLearningdeconvolutionnetwork2015,Yasrab_2017} and the increase in annotated data \cite{everingham2010pascal,HariharanSemanticcontoursinverse2011,zhou2017scene,Temel2017_NIPSW,Temel2018_SPM,Temel2019_VIP,Temel2018_CUREOR}.  
As progress continues to improve, these models will leave labs and become more prevalent in real-world scenarios. 
The tasks of out-of-distribution (OOD) detection \cite{hendrycks17baseline, Hendrycks2018DeepAD} and error calibration \cite{naeini2015obtaining, guo2017calibration} are core pillars to designing robust models for real-world application.
Despite the progress for improving the robustness of deep image classifiers, there is no such effort with semantic segmentation.
It is then important to improve understanding of the failure modes of semantic segmentation models in order to develop a principled approach to robust design. 

In this paper, we will adapt the image classification techniques for OOD detection and error calibration to semantic segmentation.  Also, we will provide evidence of a possible flaw in the classifier design that exists in many state-of-the-art semantic segmentation models. In particular, we will demonstrate the effects of non-distinction arising from $softmax$ that occur between the background class and another class. Also, we will demonstrate how to reduce those effects by restricting how the background is classified. 

\textbf{Properties of Softmax}: A semantic segmentation model, $\M$, is a classification model performed at each RGB-pixel of an image, $\vect{x}\in\R^{3\times W\times H}$, where $\M: \R^{3\times W\times H} \rightarrow \R^{k\times W\times H}$. When optimized with cross-entropy each RGB-pixel is mapped to a $k$-dimensional representation space, $\K = \R^k$, as a vector, $\vect{v}\in\mathbf{K}$. In-sample pixels, $\vect{x}[n,m]\leftrightarrow c$, correspond to a class label $c$ within the set of class labels, $\mathbf{C} = \{c: c =0,1,\dots,k-1\} $. Each pixel is classified by selecting the maximum component, $\argmax_{\vect{v}} \vect{v} = c $.   The softmax operator \eqref{eq:softmax}, which is used during optimization, is surjective-only from $\K$ onto the interior of the $(k-1)$-simplex. To simplify notation, let a vector of exponentials be $\e^{\vect{v}}= \irow{e^{v_0}& e^{v_1}&\dots&e^{v_{k-1}}}$.
\begin{equation}
\label{eq:softmax}
\begin{aligned}
\softmax :\K\rightarrow\S^{k-1}\\
\softmax(\vect{v}) = \frac{\e^{\vect{v}}}{\sum^{k-1}_{i=0} e^{v_i}} 
\end{aligned}
\end{equation}

\begin{lemma}
\label{lem:surjection}
For $\softmax$ to be a surjection, it must be true that for every $\vect{s}\in \S^{k-1}$, there is a $\vect{v}\in\K$ such that $\softmax(\vect{v})=\vect{s}$.
\end{lemma}
\begin{proof}

We can rewrite eq. \eqref{eq:softmax} as
\begin{align*}
\vect{v} - \logsumexp{\vect{v}} &= \log\vect{s}
\end{align*}
By taking the $\log$ of the equality, components of the right hand side now have a new range $\log \vect{s}\in(-\infty, 0]^k$.
On the left hand side of the equality, the operator $\logsumexp{\vect{v}}$ can be thought of as a smooth version of $\max{\vect{v}}$, which has the property, $\logsumexp{\vect{v}} \geq \max{\vect{v}}$. Though exchanging $\max$ and $\logsumexp$ is not essential to the proof, it simplifies the argument that $\logsumexp$ restricts the range of the left hand side to the negative orthant of $\K$. It is then equivalent to let $\vect{w} = \vect{v} - \logsumexp{\vect{v}}$ if $\vect{w}\in\K$, $\vect{w}\in(-\infty,0]^k$.  Finally, $\vect{w} = \log\vect{s}$, therefore $\softmax$ is surjective.
\end{proof}
There are a two cases where $\softmax$ is not distinctive: all components are equal and some components are equal while all others approach $-\infty$. 
\begin{lemma}
\label{lem:noinj1}
Given any two vectors, $\vect{v}, \vect{w} \in \K$,  where $\vect{v}\not= \vect{w}$, $v_i=v_j$ and $w_i=w_j \forall i,j\in[0,k-1]$,  we claim that $\softmax(\vect{v}) = \softmax(\vect{w})$.
\end{lemma}

\begin{proof}
Let $\mathbf{1} = \irow{1&1&\dots&1}$ and $\e^{\vect{v}}= \irow{e^{v_0}& e^{v_1}&\dots&e^{v_{k-1}}}$
\begin{align*}
\softmax(\vect{v}) = \frac{\e^{\vect{v}}}{\langle \e^{\vect{v}}, \mathbf{1} \rangle}  &=\frac{\e^{\vect{w}}}{\langle \e^{\vect{w}}, \mathbf{1} \rangle} = \softmax(\vect{w})\\
\e^{\vect{v}}\langle \e^{\vect{w}}, \mathbf{1} \rangle  &=\e^{\vect{w}}\langle \e^{\vect{v}}, \mathbf{1} \rangle\\
(\e^{\vect{v}} \otimes \e^{\vect{w}}) \mathbf{1}  &=(\e^{\vect{w}}\otimes \e^{\vect{v}}) \mathbf{1}
\end{align*}
which implies the matrix exponential $e^A = \e^{\vect{v}}\otimes \e^{\vect{w}}$ is symmetric,
\begin{equation}
\label{eq:surjection}
A = \begin{bmatrix}
v_0+w_0 & v_0+w_1 & \cdots & v_0 + w_{k-1} \\
v_1+w_0 & v_1+w_1 & \cdots & v_1 + w_{k-1} \\
\vdots & \vdots & \ddots & \vdots \\
v_{k-1}+w_0 & v_{k-1} + w_1 & \cdots & v_{k-1} + w_{k-1} \\
\end{bmatrix} \\
\end{equation}
and when combined with the fact $\vect{v} \not = \vect{w}$, the entries of $\vect{v}$ and $\vect{w}$ must follow 
\begin{align*}
v_0 &= v_1 = \dots = v_{k-1}\\
w_0 &= w_1 = \dots = w_{k-1}
\end{align*}
Therefore, $\softmax$ is not injective.
\end{proof}

\begin{lemma}
\label{lem:noinj2}
Given any two vectors, $\vect{v}, \vect{w} \in \K$ of the form $\irow{z&\dots&z&x_n&\dots&x_m&z&\dots&z}^T$, where $\vect{v}\not= \vect{w}$, $x_i=x_j \forall i,j\in[n,m]$, we claim that $\lim\limits_{z\rightarrow-\infty}\softmax(\vect{v}) = \softmax(\vect{w})$.
\end{lemma}
\begin{proof}
We begin by taking the limit, 
\begin{align*}
    \lim\limits_{z\rightarrow-\infty} \e^{\vect{v}} = \irow{0\dots&0&e^{x_n}&\dots&e^{x_m}&0&\dots0}^T
\end{align*}
Then from \eqref{eq:surjection}, we only need to show that the symmetry $e^A = e^{A^T}$ holds. The matrix exponential $\lim\limits_{z\rightarrow-\infty}e^{A}$ only has non-zero entries at $i,j\in[n,m]$, which is only symmetric when,
\begin{align*}
v_n &= v_{n+1} = \dots = v_m\\
w_n &= w_{n+1} = \dots = w_m
\end{align*}
Therefore, $\lim\limits_{z\rightarrow-\infty}\softmax(\vect{v})$ is not injective.
\end{proof}
Given the proof from Lemmas \ref{lem:surjection}, \ref{lem:noinj1}, and \ref{lem:noinj2}, $\softmax$ is surjective and not distinct in the domain of $\softmax$, thereby, losing information when two or more components approach equivalence.  We will exploit this fact to constrain our model to predict the background class only when all in-domain components lay in the negative-orthant.


\label{sec:bgest}

\begin{figure}[t]
\hspace*{-.4cm}   
    \centering
\begin{tabular}{rcc}
   %
   \includegraphics[width=\linewidth]{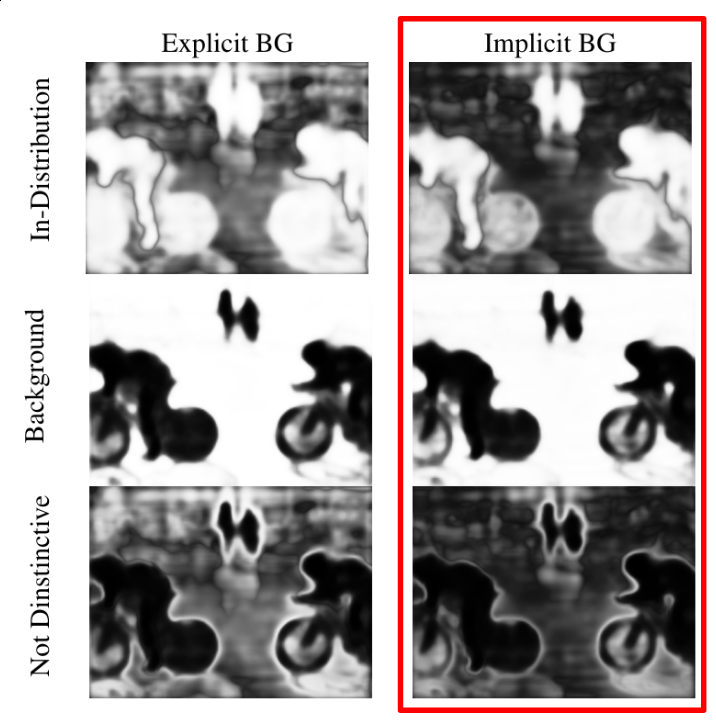}
\end{tabular}
    \caption{Comparison of membership maps of an unmodified model that has explicit background estimation and the modified model that implicitly estimates the background.  The product of the In-Distribution and Background membership maps result in the Not Distinctive visualization. It is apparent that the intensity of non-distinctiveness is greater for explicit background estimation than implicit background estimation.}
\label{fig:confmap}%
\end{figure}

\section{Method}
\textbf{Non Distinctiveness}: The necessity to consider the background as a class for semantic segmentation is a consequence of not densely annotating every pixel in each image. 
The current state-of-the-art semantic segmentation models treat background as a separate class and dedicate model parameters for use in its estimation \cite{DBLP:journals/corr/ChenPK0Y16,DBLP:journals/corr/ChenPSA17,chen2018deeplab, chen2018encoder, NohLearningdeconvolutionnetwork2015,Yasrab_2017}. 
To better understand the effects of not preserving distinctiveness we propose the following method for measuring distinctiveness. Consider the following construction of the indices of a pixel in representation space, $\K$, or on the simplex, $\S^{k-1}$:
\begin{equation}
\label{eq:index}
\begin{aligned}
\vect{x}&= \irow{x_{BG}&x_{1}&x_{2}&\cdots&x_{k}}^T\\
\vect{x}_{ID}&= \irow{x_{1}&x_{2}&\cdots&x_{k}}^T\\
\end{aligned}
\end{equation}
We can use the indexings on $\vect{v}$ and $\softmax$ from \eqref{eq:index} to define membership indicators for in-distribution, background, and non-distinct:  
\begin{equation}
\label{eq:nodist}
\begin{aligned}
\bm{\mu}_{ID}(\vect{v}) &= \max\softmax(\vect{v}_{ID})\\
\bm{\mu}_{BG}(\vect{v}) &= \softmax(\vect{v})_{BG}\\
\bm{\mu}_{ND}(\vect{v}) &= \mathbf{\mu}_{BG}\mathbf{\mu}_{ID}
\end{aligned}
\end{equation}
We visualize the output of \eqref{eq:nodist} in Fig. \ref{fig:confmap}, where it becomes clear that the key role estimating the background class plays is to define the spatial edges of in-distribution predictions by suppressing errant in-distribution components.
This action can also be considered to be an out-of-distribution (OOD) detector as it is explicitly estimating the complement of $\vect{v}_{ID}$. 
Though this does not appear to effect the overall performance of the model in the semantic segmentation task, the non-distinctiveness between background and in-distribution components can lead to robustness issues.

\textbf{Out-of-distribution Detection}: The state-of-the-art OOD detection methods train on one in-distribution dataset and evaluate on an out-of-distribution dataset \cite{hendrycks17baseline, Hendrycks2018DeepAD}. An important distinction between semantic segmentation and image classification tasks, is that it is not necessary in image classification to have a separate class for "background" as the annotations for in-distribution images are at the image level vice pixel level. Due to this difference, we evaluate the background component of the semantic segmentation models as built-in OOD detector.

\textbf{Calibration Error}: 
Empirical observations made by \cite{guo2017calibration, HintonDistillingknowledgeneural2015} have shown that a consequence of $\softmax$ is the tendency for models to map in-distribution inputs as large magnitudes in $\K$ thereby assigning membership very close to a $\S^0$ face, which causes the model to lack confidence calibration.  They also demonstrated that calibration can be improved with use of temperature scaling, which re-scaled the operand for $\softmax$.  In effect, temperature scaling reduced the effects discussed in the proof for Lemma 3.3, allowing for membership of higher-order faces of $\S^{k-1}$ to emerge more often. We will study the effects of implicit background estimation on model calibration with Expected Calibration Error (ECE) as formulated by \cite{naeini2015obtaining, guo2017calibration}.

\begin{figure}[t]
\includegraphics[width=0.9\linewidth]{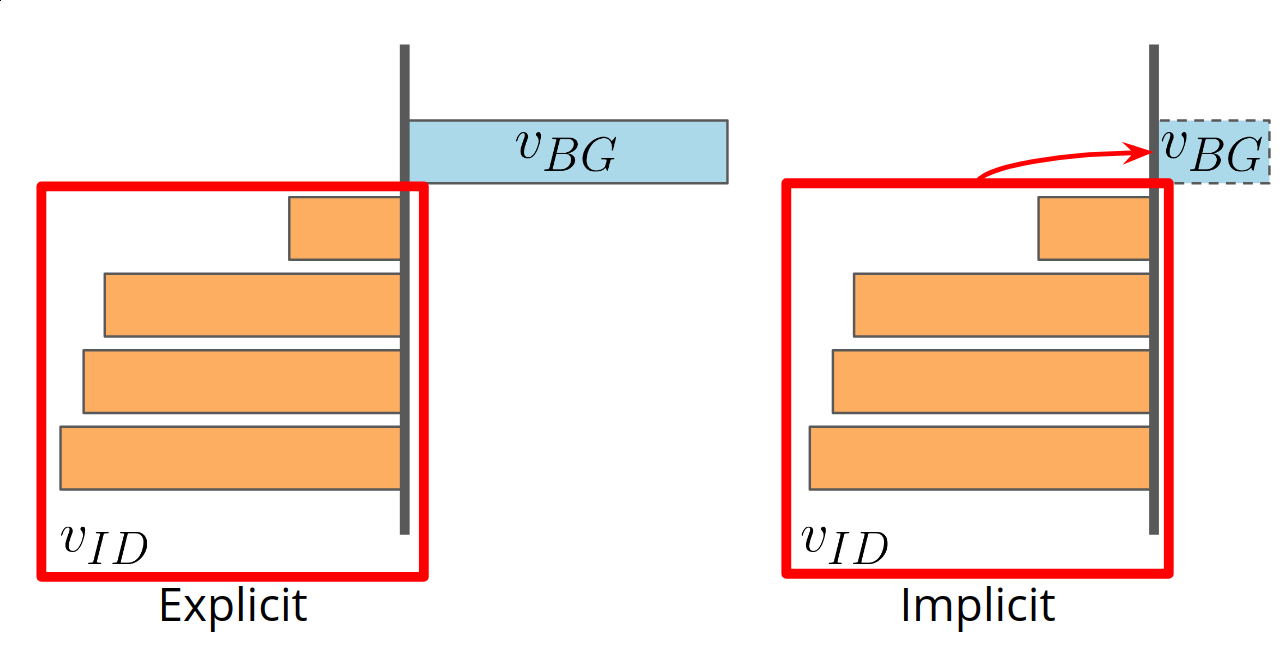}
\centering
\caption{The bars represent magnitudes of $\vect{v}\in \K^5$.  Bars enclosed by the red box are the in-distrubution components, $\vect{v}_{ID}$, while the $\vect{v}_{BG}$ is the background component.  The plot labeled Explicit has $\vect{v}_{BG}$ as an independent component, while Implicit has $\vect{v}_{BG}$ depend on $-\logsumexp(\vect{v}_{ID})$.}
\label{fig:impexpbg}
\end{figure}

\textbf{Implicit Background Estimation}: To investigate the  effects of non-distinctiveness of background and in-distribution components we propose a method of implicitly estimating background. Our method imposes a positive orthant constraint on detection for any class by parameterizing the background component with the composite vector shown in \eqref{eq:impbg}.

To demonstrate that, by implicitly estimating the background, there is a reduction in the possible non-distinctive cases demonstrated in the proofs for Lemmas \ref{lem:noinj1}, \ref{lem:noinj2}, we provide the following:
\begin{theorem}
There exists some $\hat{\K}\subset\K$ where $\max \vect{v} \ge 0\;$ for all $\vect{v}\in\hat{\K}$.  Furthermore, such a set exists while preserving the domain of each component. Finally, any $\vect{v}\in\hat{\K}$ with an preserved domain will reduce the number of non-distinctive cases.
\label{thm:posorth}
\end{theorem}
\begin{proof}
Consider the following operation on the composite vector from \eqref{eq:index}, 
\begin{equation}
\begin{aligned}
\label{eq:impbg}
v_{BG} &= -\logsumexp (\vect{v}_{ID})\\
\vect{v}_{ImpBG} &= \irow{v_{BG}&v_{ID_1}&v_{ID_2}&\cdots&v_{ID_{k-1}}}^T\\
\end{aligned}
\end{equation}
To demonstrate the preservation of domain on each $v_i$, we inspect $\vect{v}_{ImpBG}$ under several limits,
\begin{equation}
\begin{aligned}
\lim\limits_{v_{ID_{i}}\rightarrow-\infty} \vect{v}_{ImpBG} &\implies v_{BG}\rightarrow +\infty\\
\lim\limits_{v_{ID_{i}}\rightarrow 0} \vect{v}_{ImpBG} &\implies v_{BG}\rightarrow 0\\
\lim\limits_{v_{ID_{i}}\rightarrow +\infty} \vect{v}_{ImpBG} &\implies v_{BG}\rightarrow -\infty\\
\end{aligned}
\end{equation}
Therefore, $\max \vect{v}_{ImpBG} \geq 0\quad \forall\vect{v}_{ID}\in \R^{k-1}$, which implies that $\vect{v}_{ImpBG}\in \hat{\K}$.
\end{proof}
The proof from Theorem \ref{thm:posorth}, demonstrates that implicit background estimation enforces $M:\R^{3\times W\times H} \rightarrow \hat{\K}$. We will experimentally verify that training under these constraints do not appreciably affect the performance of semantic segmentation, but improve performance in areas of robustness.

\section{Experiment}
\label{sec:exp}

\begin{figure}[H]
    \centering
    \subfloat[Input]{{\includegraphics[width=0.42\linewidth]{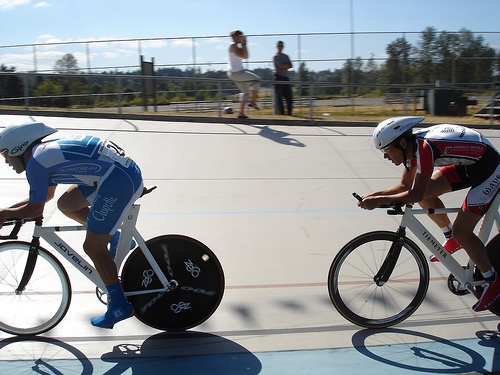}}}%
    \quad
    \subfloat[Ground Truth]{{\includegraphics[width=0.42\linewidth]{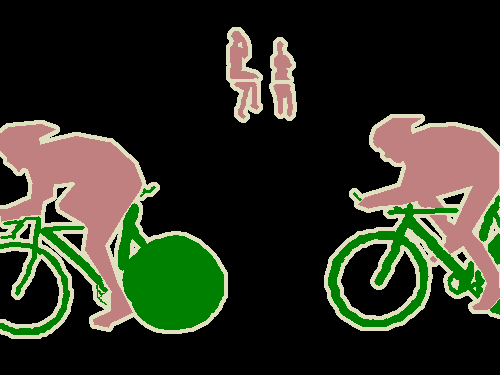} }}%
    \\
    \subfloat[Explicit BG]{{\includegraphics[width=0.42\linewidth]{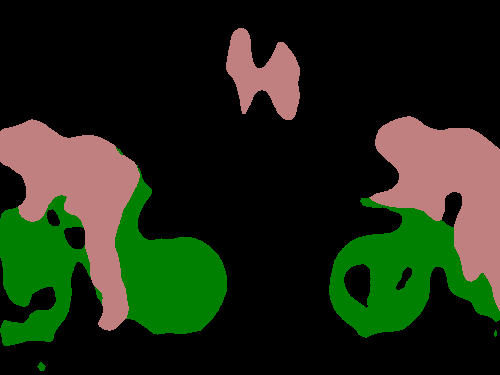}}}%
    \quad
    \subfloat[Implicit BG]{{\includegraphics[width=0.42\linewidth]{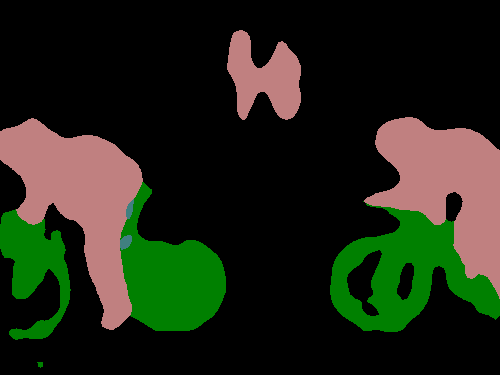} }}%
\caption{Visual comparison between explicit and implicit background estimation on semantic segmentation task using image 2007\_007836 from PASCAL VOC.}
\label{fig:seg}%
\end{figure}

\textbf{Semantic Segmentation}: The PASCAL VOC 2012 \cite{everingham2010pascal} dataset is used for semantic segmentation evaluation due to its long-standing acceptance as a benchmark and because the "background" annotations make up close to $73\% $ of the pixels in the training set, which is useful for studying the effects of non-distinction.  The dataset consists of 20 object classes, 1464 training images, 1449 validation images, and 1456 test images. Ground-truth annotations are provided with the train and validation sets, while the test set is reserved by the PASCAL VOC evaluation server.  Evaluation of semantic segmentation performance is reported in terms of mean intersection-over-union (mIOU) on the PASCAL VOC 2012 \textit{validation} and \textit{test} sets. The results reported in Table \ref{tab_semsegres}, demonstrate that the implicit background estimation has minimal impact on the semantic segmentation performance. A visual comparison of results are visualized on Fig. \ref{fig:seg}.
\begin{table}[H]
\centering
\begin{tabular}{r|cc}
\toprule
\textbf{Model}&\textbf{val} & \textbf{test}\\ \toprule
Explicit BG \cite{chen2018encoder} & 78.85 & \longdash[2]\\
Explicit BG (reproduced) & 79.96 & \textbf{76.79}\\ \hline
Implicit BG & \textbf{80.06} & 76.44\\ 
\hline
\end{tabular}
\caption{Results in terms of mIOU on PASCAL VOC 2012 \textit{val} set using DeepLabv3+ with ResNet-101 backbone.  We provide the results reported in \cite{chen2018encoder} and our reproduced results in addition to our modified implicit background estimation method. The effect of the implicit background modification is minimal. Higher is better.}
\label{tab_semsegres}
\end{table}

\textbf{Model}: For evaluation we use two versions of DeepLabV3+ with a ResNet-101 \cite{chen2018encoder} backbone as our base model and make two variants: unmodified and with the modification for implicit background estimation. The ResNet-101 backbone is the pretrained model provided in the torchvision library \cite{paszke2017automatic}.

\textbf{Training}: The models were trained on PASCAL VOC 2012 \textit{train} set with the augmented annotations from \cite{HariharanSemanticcontoursinverse2011}, totalling in 10582 \textit{train\_aug} images. This is followed by fine-tune training on PASCAL VOC \textit{train}. Both steps use the same parameters as \cite{chen2018encoder}. 

\begin{table}[H]
\centering
\begin{tabular}{r|cc}
\toprule
\textbf{Model} & \textbf{DTD} & \textbf{Noise} \\ \toprule
Explicit BG & 78.90 & 99.94 \\ \hline
Implicit BG & \textbf{82.46} & \textbf{100} \\ 
\hline
\end{tabular}
\caption{Results for the out-of-distribution detection on the Describable Texture Dataset and generated Gaussian White Noise. Implicit background estimation clearly out performs the unmodified model. Higher is better.}
\label{tab:ood}
\end{table}

\textbf{Out-of-distribution Detection}: The Describable Texture Dataset (DTD) \cite{cimpoi14describing} and a set of generated Gaussian White Noise (GWN) images are utilized for evaluating the out-of-the-box performance of detecting OOD inputs. In \cite{hendrycks17baseline, Hendrycks2018DeepAD}, area under the receiver operator characteristic (AUROC) curve was used in reporting for detecting OOD inputs for image classification. This was necessary for image classification tasks in order to evaluate at all possible decision boundaries.  However, to evaluate semantic segmentation models this is unnecessary since they are trained to directly classify background class at each pixel and already have an established decision boundary. Therefore, we will be reporting mIOU for the background class for the OOD detection task. The results reported in Table \ref{tab:ood} demonstrate that implicit background estimation yields improved robustness in the OOD detection task.

\begin{table}[H]
\centering
\begin{tabular}{r|ccc}
\toprule
\textbf{Model} & \textbf{VOC} & \textbf{DTD} & \textbf{Noise} \\ \toprule
Explicit BG & 6.02 & 9.15 & 7.02\\ 
\hline
Implicit BG & \textbf{4.58} & \textbf{4.83} & \textbf{6.52}  \\
\hline
\end{tabular}
\caption{Results for Expected Calibration Error on PASCAL VOC 2012, Describable Texture Dataset, and generated Gaussian White Noise.  Implicit background estimation improves model calibration across each dataset. Lower is better.}
\label{tab_ece}
\end{table}
\textbf{Model Calibration}: A calibrated model produces confidence scores that are the same as the expected accuracy. For example, all predictions with confidence $50\%$ should have $50\%$ accuracy.  Evaluation of model calibration is reported in terms of Expected Calibration Error (ECE) \cite{naeini2015obtaining, guo2017calibration}. ECE for both the unmodified and modified models are evaluated on PASCAL VOC 2012, DTD, and generated GWN. Results reported in Table \ref{tab_ece} clearly show that implicitly estimating the background improves model calibration without implementing the aforementioned temperature scaling.

\begin{table}[H]
\centering
\begin{tabular}{r|ccc}
\toprule
\textbf{Model} & \textbf{VOC} & \textbf{DTD} & \textbf{Noise} \\ \toprule
Explicit BG & 25.21 & 35.39 & 52.93\\ \hline
Implicit BG & \textbf{17.99} & \textbf{28.5} & \textbf{37.18} \\ 
\hline
\end{tabular}
\caption{Results of Expected Non-Distinctiveness are reported on the PASCAL VOC 2012, Describable Texture Dataset, and generated Gaussian White Noise data.  The implicit background modification out-performs the unmodified model.}
\label{tab:end}
\end{table}
\textbf{Expected Non-Distinctiveness}:  We also evaluate both models on the expected non-distinctiveness for the PASCAL VOC 2012 \textit{validation}, DTD, and GWN datasets.  This is accomplished by computing $\mathbf{E}[\bm{\mu}_{ND}(\vect{v})]$ for each dataset.  Results reported on Table \ref{tab:end} show that implicit background significantly reduces the amount of non-distinctiveness, thereby providing evidence of the analytic results from Theorem \ref{thm:posorth}.

\section{Conclusion}
We have provided both analytical and empirical evidence that implicit background estimation improves the robustness of a deep semantic segmentation network by limiting the non-distinctive mappings onto the domain of $\softmax$. Also, the increase in robustness comes without significantly affecting performance in the semantic segmentation task.  Finally, implementing implicit background estimation on any semantic segmentation model can be accomplished in about three lines of code.

\newpage
\bibliographystyle{IEEEtran}
\bibliography{refs}

\begin{thebibliography}{10}
\providecommand{\url}[1]{#1}
\csname url@samestyle\endcsname
\providecommand{\newblock}{\relax}
\providecommand{\bibinfo}[2]{#2}
\providecommand{\BIBentrySTDinterwordspacing}{\spaceskip=0pt\relax}
\providecommand{\BIBentryALTinterwordstretchfactor}{4}
\providecommand{\BIBentryALTinterwordspacing}{\spaceskip=\fontdimen2\font plus
\BIBentryALTinterwordstretchfactor\fontdimen3\font minus
  \fontdimen4\font\relax}
\providecommand{\BIBforeignlanguage}[2]{{%
\expandafter\ifx\csname l@#1\endcsname\relax
\typeout{** WARNING: IEEEtran.bst: No hyphenation pattern has been}%
\typeout{** loaded for the language `#1'. Using the pattern for}%
\typeout{** the default language instead.}%
\else
\language=\csname l@#1\endcsname
\fi
#2}}
\providecommand{\BIBdecl}{\relax}
\BIBdecl

\bibitem{DBLP:journals/corr/ChenPK0Y16}
\BIBentryALTinterwordspacing
L.~Chen, G.~Papandreou, I.~Kokkinos, K.~Murphy, and A.~L. Yuille, ``Deeplab:
  Semantic image segmentation with deep convolutional nets, atrous convolution,
  and fully connected crfs,'' \emph{CoRR}, vol. abs/1606.00915, 2016. [Online].
  Available: \url{http://arxiv.org/abs/1606.00915}
\BIBentrySTDinterwordspacing

\bibitem{DBLP:journals/corr/ChenPSA17}
\BIBentryALTinterwordspacing
L.~Chen, G.~Papandreou, F.~Schroff, and H.~Adam, ``Rethinking atrous
  convolution for semantic image segmentation,'' \emph{CoRR}, vol.
  abs/1706.05587, 2017. [Online]. Available:
  \url{http://arxiv.org/abs/1706.05587}
\BIBentrySTDinterwordspacing

\bibitem{chen2018deeplab}
L.-C. Chen, G.~Papandreou, I.~Kokkinos, K.~Murphy, and A.~L. Yuille, ``Deeplab:
  Semantic image segmentation with deep convolutional nets, atrous convolution,
  and fully connected crfs,'' \emph{IEEE transactions on pattern analysis and
  machine intelligence}, vol.~40, no.~4, pp. 834--848, 2018.

\bibitem{chen2018encoder}
L.-C. Chen, Y.~Zhu, G.~Papandreou, F.~Schroff, and H.~Adam, ``Encoder-decoder
  with atrous separable convolution for semantic image segmentation,''
  \emph{arXiv preprint arXiv:1802.02611}, 2018.

\bibitem{NohLearningdeconvolutionnetwork2015}
H.~Noh, S.~Hong, and B.~Han, ``Learning deconvolution network for semantic
  segmentation,'' in \emph{Proceedings of the {{IEEE International Conference}}
  on {{Computer Vision}}}, 2015, pp. 1520--1528.

\bibitem{Yasrab_2017}
\BIBentryALTinterwordspacing
R.~Yasrab, N.~Gu, and X.~Zhang, ``An encoder-decoder based convolution neural
  network (cnn) for future advanced driver assistance system (adas),''
  \emph{Applied Sciences}, vol.~7, no.~4, p. 312, Mar 2017. [Online].
  Available: \url{http://dx.doi.org/10.3390/app7040312}
\BIBentrySTDinterwordspacing

\bibitem{everingham2010pascal}
M.~Everingham, L.~Van~Gool, C.~K. Williams, J.~Winn, and A.~Zisserman, ``The
  pascal visual object classes (voc) challenge,'' \emph{International journal
  of computer vision}, vol.~88, no.~2, pp. 303--338, 2010.

\bibitem{HariharanSemanticcontoursinverse2011}
\BIBentryALTinterwordspacing
B.~Hariharan, P.~Arbelaez, L.~Bourdev, S.~Maji, and J.~Malik, ``Semantic
  contours from inverse detectors,'' in \emph{2011 {{International Conference}}
  on {{Computer Vision}}}.\hskip 1em plus 0.5em minus 0.4em\relax {IEEE}, 2011,
  pp. 991--998. [Online]. Available:
  \url{http://ieeexplore.ieee.org/document/6126343/}
\BIBentrySTDinterwordspacing

\bibitem{zhou2017scene}
B.~Zhou, H.~Zhao, X.~Puig, S.~Fidler, A.~Barriuso, and A.~Torralba, ``Scene
  parsing through ade20k dataset,'' in \emph{Proceedings of the IEEE Conference
  on Computer Vision and Pattern Recognition}, 2017.

\bibitem{Temel2017_NIPSW}
D.~Temel, G.~Kwon, M.~Prabhushankar, and G.~AlRegib, ``{CURE-TSR: Challenging
  Unreal and Real Environments for Traffic Sign Recognition},'' in \emph{Neural
  Information Processing Systems (NeurIPS), Machine Learning for Intelligent
  Transportation Systems Workshop}, 2017.

\bibitem{Temel2018_SPM}
D.~{Temel} and G.~{AlRegib}, ``Traffic signs in the wild: Highlights from the
  ieee video and image processing cup 2017 student competition [{SP}
  competitions],'' \emph{IEEE Sig. Proc. Mag.}, vol.~35, no.~2, pp. 154--161,
  March 2018.

\bibitem{Temel2019_VIP}
D.~Temel, T.~Alshawi, M.-H. Chen, and G.~AlRegib, ``Challenging environments
  for traffic sign detection: Reliability assessment under inclement
  conditions,'' \emph{arXiv:1902.06857}, 2019.

\bibitem{Temel2018_CUREOR}
D.~Temel, J.~Lee, and G.~AlRegib, ``{CURE-OR: Challenging Unreal and Real
  Environments for Object Recognition},'' in \emph{IEEE International
  Conference on Machine Learning and Applications (ICMLA)}, 2018.

\bibitem{hendrycks17baseline}
D.~Hendrycks and K.~Gimpel, ``A baseline for detecting misclassified and
  out-of-distribution examples in neural networks,'' \emph{Proceedings of
  International Conference on Learning Representations}, 2017.

\bibitem{Hendrycks2018DeepAD}
D.~Hendrycks, M.~Mazeika, and T.~G. Dietterich, ``Deep anomaly detection with
  outlier exposure,'' \emph{CoRR}, vol. abs/1812.04606, 2018.

\bibitem{naeini2015obtaining}
M.~P. Naeini, G.~Cooper, and M.~Hauskrecht, ``Obtaining well calibrated
  probabilities using bayesian binning,'' in \emph{Twenty-Ninth AAAI Conference
  on Artificial Intelligence}, 2015.

\bibitem{guo2017calibration}
C.~Guo, G.~Pleiss, Y.~Sun, and K.~Q. Weinberger, ``On calibration of modern
  neural networks,'' in \emph{International Conference on Machine Learning},
  2017, pp. 1321--1330.

\bibitem{HintonDistillingknowledgeneural2015}
\BIBentryALTinterwordspacing
G.~Hinton, O.~Vinyals, and J.~Dean, ``Distilling the knowledge in a neural
  network,'' in \emph{NIPS Deep Learning and Representation Learning Workshop},
  2015. [Online]. Available: \url{http://arxiv.org/abs/1503.02531}
\BIBentrySTDinterwordspacing

\bibitem{paszke2017automatic}
A.~Paszke, S.~Gross, S.~Chintala, G.~Chanan, E.~Yang, Z.~DeVito, Z.~Lin,
  A.~Desmaison, L.~Antiga, and A.~Lerer, ``Automatic differentiation in
  pytorch,'' in \emph{NIPS 2017 Autodiff Workshop: The Future of Gradient-based
  Machine Learning Software and Techniques, Long Beach, CA, US, December 9,
  2017}, 2017.

\bibitem{cimpoi14describing}
M.~Cimpoi, S.~Maji, I.~Kokkinos, S.~Mohamed, , and A.~Vedaldi, ``Describing
  textures in the wild,'' in \emph{Proceedings of the {IEEE} Conf. on Computer
  Vision and Pattern Recognition ({CVPR})}, 2014.

\end{thebibliography}
\end{document}